\documentclass[letterpaper, 10 pt, conference]{ieeeconf}
\IEEEoverridecommandlockouts    
\usepackage[english]{babel}
\usepackage{graphicx} 
\graphicspath{{ {./images/} }}
\usepackage{epsfig} 
\usepackage{amsmath} 
\usepackage{amssymb}  
\usepackage{color}
\usepackage{subcaption}
\usepackage{caption}
\usepackage{cite}
\newtheorem{theorem}{Theorem}[section]
\newtheorem{proposition}[theorem]{Proposition}
\newtheorem{corollary}[theorem]{Corollary}

\newtheorem{remark}{Remark}[section]

\usepackage{hyperref}
\usepackage{cleveref}

\begin{document}

\title{Port-Hamiltonian Koopman Operator Synthesis for Mechanical Systems
\author{ Rajpal Singh, Aditya Singh and Jishnu Keshavan}
\thanks{The authors are with the Department of Mechanical Engineering, Indian Institute of Science, Bangalore, Karnataka~560012, India 
    ({\tt\footnotesize{email: rajpalsingh@iisc.ac.in, adityasingh2@iisc.ac.in, kjishnu@iisc.ac.in}}).
    }
}

\maketitle
\IEEEpeerreviewmaketitle

\begin{abstract}
Finite-dimensional Koopman models enable efficient linear prediction and control of nonlinear robotic systems. However, models learned purely from trajectory data may violate the energetic structure of the underlying mechanics, producing predictions that exhibit artificial energy growth and diverge under recursive propagation. This work presents a structure-preserving Koopman framework for Euler-Lagrange systems built on generalized-momentum coordinates. The momentum transformation exposes the mechanical actuation as a known, state-independent port, which is preserved explicitly in the lifted dynamics. A structure-constrained neural architecture is developed to jointly learn the lifting functions and a port-Hamiltonian Koopman generator, rendering the learned dynamics passive by construction rather than through penalty terms or post-hoc projection. A Cayley-midpoint discretization further preserves the corresponding storage-dissipation balance exactly in discrete time. These properties are established analytically by deriving the discrete storage balance and associated stability guarantees of the learned predictor. Simulation and experimental studies demonstrate improved prediction accuracy, data efficiency, and closed-loop tracking over Koopman baselines,  with increasing gains for higher-dimensional systems.


\end{abstract}

\section{Introduction}
\label{sec:introduction}
Robotic systems are governed by nonlinear dynamics, which makes the design of accurate and computationally efficient prediction and control methods challenging~\cite{craig1989introduction}. While nonlinear models can capture such dynamics more faithfully, their direct use in real-time optimization-based control often incurs substantial computational cost. A common alternative is therefore to employ linear models obtained through local linearization around nominal operating points, enabling the use of efficient linear control techniques. However, such approximations are inherently local and lose accuracy away from the operating point about which they are constructed.

Koopman operator theory provides an alternative by representing nonlinear dynamics through the linear evolution of suitably chosen observable functions~\cite{koopman1931hamiltonian}. This enables nonlinear systems to be described by linear dynamics in a lifted space without restricting the representation to a local neighborhood. Although the exact Koopman operator is generally infinite-dimensional, finite-dimensional approximations can be constructed from data using predefined dictionaries or learned lifting functions~\cite{edmdwilliams2014, sindybrunton2016, champion2019data}. The resulting lifted linear predictors are attractive for robotic control~\cite{sah2024overview,shi2026koopman}, as they retain nonlinear modeling capability while remaining compatible with computationally efficient linear control frameworks such as Model Predictive Control (MPC).

A central limitation of finite-dimensional Koopman learning is that predictive accuracy does not guarantee physically consistent dynamics. A model may match short-horizon trajectories while still containing spurious unstable modes or exhibiting artificial energy growth under recursive prediction. Recent works have therefore introduced stability and dissipativity constraints into Koopman learning~\cite{hara2020learning,10091950,9867865,sakibNoise,yu2025dissapative}. The method in~\cite{hara2020learning} imposes dissipativity using prior system knowledge, while~\cite{10091950} derives stability conditions and uses predefined basis functions to construct stable models. Similarly,~\cite{9867865,sakibNoise} enforce stability through stable linear parameterizations, with~\cite{9867865} jointly learning the lifting functions and~\cite{sakibNoise} using a predefined Polyflow basis. More recently,~\cite{yu2025dissapative} stabilizes an unconstrained neural Koopman model through minimal post-hoc correction. However, these approaches do not explicitly preserve the energetic structure of mechanical systems. A stable lifted operator need not distinguish conservative energy exchange from physical dissipation or retain the known actuation structure. This motivates a Koopman formulation that embeds these mechanical properties directly into the learned dynamics.

Port-Hamiltonian theory~\cite{porthamschaft} provides a natural framework for this purpose by separating energy storage, conservative interconnection, dissipation, and external ports. Recent work has established an operator-level Koopman decomposition for port-Hamiltonian systems through projection onto a prescribed observable space ~\cite{preciado2026koopman}. While this provides an important connection between Koopman and port-Hamiltonian representations, it does not directly address the data-driven synthesis of a finite-dimensional controlled Koopman model in which the lifting functions, structured generator, and known mechanical actuation are learned or retained jointly. This distinction is particularly important for Euler-Lagrange systems. In conventional position-velocity coordinates, the control channel is configuration dependent through the inertia matrix, so the known mechanical actuation does not naturally appear as a constant input port in a linear lifted model. Moreover, the continuous-time port-Hamiltonian structure must survive the discretization required for prediction and control. These considerations motivate a formulation that simultaneously learns the lifted dynamics, preserves the known actuation port, and retains the associated dissipativity properties in discrete time.

To address these challenges, we formulate Koopman learning in generalized-momentum coordinates~\cite{singh2025generalized}. This transformation absorbs the configuration-dependent inertial contribution into the autonomous dynamics while exposing the known generalized-force distribution as a state-independent actuation port. The resulting coordinates are also naturally compatible with the Hamiltonian description of mechanical energy. Building on this representation, we develop a structure-constrained neural architecture that jointly learns the lifting functions and a port-Hamiltonian Koopman generator while preserving the known actuation port. The generator is parameterized to explicitly encode storage, conservative interconnection, and dissipation in the learned lifted dynamics. Since the resulting predictor is ultimately used in discrete time, we further employ a Cayley-midpoint discretization that preserves the associated storage-dissipation properties of the continuous-time model. The framework is evaluated on serial manipulators of increasing dimension and on a Franka FR3 through open-loop prediction, data efficiency, spectral stability, and closed-loop tracking studies.

The main contributions of this work are as follows:
\begin{itemize}

    \item We formulate a finite-dimensional port-Hamiltonian Koopman representation for controlled Euler-Lagrange systems in generalized-momentum coordinates, which exposes and preserves the known mechanical actuation as a state-independent input port.
    
    \item We introduce a structure-constrained neural architecture that jointly learns the lifting map and port-Hamiltonian Koopman generator from trajectory data, enforcing storage, conservative interconnection, 
    and dissipation structure by construction.

    \item We derive a Cayley-midpoint discrete-time realization and establish exact discrete storage-dissipation balance, zero-input contractivity, and stability guarantees.

    \item We evaluate the framework on systems on a range of serial manipulators including a Franka FR3, demonstrating improved prediction accuracy, data efficiency, spectral stability, and closed-loop tracking over Koopman baselines.
\end{itemize}

The remainder of the paper is organized as follows. Section~\ref{sec:preliminaries} presents the preliminaries. Section~\ref{sec:ph_koopman} develops the proposed port-Hamiltonian Koopman learning framework, and Section~\ref{sec:guarantees} establishes its dissipativity and stability properties. Section~\ref{sec:results} presents the simulation and hardware evaluation, followed by concluding remarks in Section~\ref{sec:conc}.

\section{Preliminaries}
\label{sec:preliminaries}

This section introduces preliminaries on the Koopman operator theory and the Euler-Lagrange dynamics. 

\subsection{Koopman Operator Theory}
\label{sec:koopman_prelim}
Consider a nonlinear autonomous system $\dot{\boldsymbol{x}} {=} \boldsymbol{f}(\boldsymbol{x})$, where $\boldsymbol{x}\in\mathbb{X}\subset\mathbb{R}^{n}$ denotes the state and $\boldsymbol{f}:\mathbb{X}\rightarrow\mathbb{R}^{n}$ is assumed to be Lipschitz continuous over $\mathbb{X}$.  For a sampling interval $h{>}0$, the corresponding discrete-time dynamics at $k^\mathrm{th}$ timestep can be written as $\boldsymbol{x}_{k+1}{=}\boldsymbol{F}_{h}(\boldsymbol{x}_{k})$, where $\boldsymbol{F}_h:\mathbb{X}\rightarrow\mathbb{X}$ denotes the flow map over one sampling interval.
As per Koopman operator theory~\cite{koopman1931hamiltonian}, the evolution of an observable $\sigma:\mathbb{X}\rightarrow\mathbb{C}$ belonging to a Hilbert space $\mathcal{H}$ is described by the infinite-dimensional linear Koopman operator $\mathcal{K}_h:\mathcal{H}\rightarrow\mathcal{H}$ as
\begin{equation}
    \mathcal{K}_h\sigma(\boldsymbol{x}_{k}) = \sigma\!\left(\boldsymbol{F}_h(\boldsymbol{x}_{k})\right) =\sigma(\boldsymbol{x}_{k+1}).
    \label{eq:koopman_operator}
\end{equation}

The Koopman framework can be extended to systems with control inputs. Consider the control-affine dynamics
\begin{equation}
    \label{eq:base_affine}
    \dot{\boldsymbol{x}} = \boldsymbol{f}_{0}(\boldsymbol{x}) + \sum_{i=1}^{m} \boldsymbol{f}_{i}(\boldsymbol{x})u_i,
\end{equation}
where
$\boldsymbol{u}=[u_1,\ldots,u_m]^{\top}\in\mathbb{R}^{m}$ represents the control inputs. $\boldsymbol{f}_{0}$ denotes the drift dynamics, and $\boldsymbol{f}_{i}$ denote the control vector fields. For a finite-dimensional lifting $\boldsymbol{z} = \boldsymbol{\Psi}(\boldsymbol{x}) =
    \begin{bmatrix} \phi_1(\boldsymbol{x}) &
        \phi_2(\boldsymbol{x}) &
        \cdots &
        \phi_{n_z}(\boldsymbol{x})
    \end{bmatrix}^{\top}
    \in\mathbb{R}^{n_z}$, controlled Koopman models are commonly expressed in bilinear or linear
control-affine forms as
\begin{align}
    \label{eq:koopman_bil_con}
    &\dot{\boldsymbol{z}}=\boldsymbol{A}_c\boldsymbol{z}+ \boldsymbol{B}_c(\boldsymbol{z} \otimes \boldsymbol{u}),\;\boldsymbol{x}=\boldsymbol{C}^{x}\boldsymbol{z},\\
    \label{eq:koopman_lin_con}
    &\dot{\boldsymbol{z}}=\boldsymbol{A}_c\boldsymbol{z}+ \boldsymbol{B}_c\boldsymbol{u},\;\boldsymbol{x}=\boldsymbol{C}^{x}\boldsymbol{z}
\end{align}
where $\boldsymbol{A}_c\in\mathbb{R}^{n_z\times n_z}$ and $\boldsymbol{C}^x\in\mathbb{R}^{n\times n_z}$. The input matrix $\boldsymbol{B}_c\in\mathbb{R}^{n_z\times n_zm}$ for the bilinear representation and $\boldsymbol{B}_c\in\mathbb{R}^{n_z\times m}$ for the linear representation. Further, $\otimes$ denotes the Kronecker product. For the linear lifted model considered in this work, the corresponding
discrete-time predictor is
\begin{align}
    \label{eq:koopman_lin_disc}
    \boldsymbol{z}_{k+1} = \boldsymbol{A}\boldsymbol{z}_{k} + \boldsymbol{B}\boldsymbol{u}_{k}, \; \boldsymbol{x}_{k+1}=\boldsymbol{C}^{x}\boldsymbol{z}_{k+1}.
\end{align}
A more comprehensive treatment can be found in~\cite{goswami2017global, singh2025adaptive}.
\subsection{Euler-Lagrange Dynamics}
An Euler-Lagrange mechanical system, with generalized coordinates $\boldsymbol{q}{\in}\mathbb{R}^{n_q}$ and control input $\boldsymbol{u}{\in}\mathbb{R}^{m}$, is described as 
\begin{equation}
    \label{eq:EL}
    \boldsymbol{M}(\boldsymbol{q})\ddot{\boldsymbol{q}} + \boldsymbol{C}(\boldsymbol{q},\dot{\boldsymbol{q}}) \dot{\boldsymbol{q}} + \boldsymbol{g}(\boldsymbol{q}) = \boldsymbol{S}_a\boldsymbol{u}, 
\end{equation}
where $\boldsymbol{M}(\boldsymbol{q}){\succ}0$ is the inertia matrix, $\boldsymbol{C}(\boldsymbol{q},\dot{\boldsymbol{q}})$ is the Coriolis matrix, $\boldsymbol{g}(\boldsymbol{q}) {=} \nabla_{\boldsymbol q}V(\boldsymbol q)$ denotes the generalized conservative forces, $V(\boldsymbol{q})$ denotes the potential energy and $\boldsymbol{S}_a\in\mathbb{R}^{n_q\times m}$ is a constant generalized-force distribution matrix. For direct full actuation, $\boldsymbol{S}_a=\boldsymbol{I}_{n_q}$, whereas $m{<}n_q$ corresponds
to an underactuated system with fewer independent control inputs than generalized coordinates. Using the conventional position-velocity state $\boldsymbol{x}_v {=}[\boldsymbol{q}^\top,\dot{\boldsymbol{q}}^\top]^\top$,
the dynamics become
\begin{equation}
    \dot{\boldsymbol{x}}_v
    {=}
    \begin{bmatrix}
        \dot{\boldsymbol{q}}\\
        -\boldsymbol{M}^{-1}(\boldsymbol{q})
        \left[
            \boldsymbol{C}(\boldsymbol{q},\dot{\boldsymbol{q}})
            \dot{\boldsymbol{q}}
            {+}
            \boldsymbol{g}(\boldsymbol{q})
        \right]
    \end{bmatrix}
    {+}
    \begin{bmatrix}
        \boldsymbol{0}\\
        \boldsymbol{M}^{-1}(\boldsymbol{q})\boldsymbol{S}_a
    \end{bmatrix}
    \boldsymbol{u}. \nonumber
    \label{eq:velocity_state}
\end{equation}

The input distribution in the position-velocity state depends explicitly on $\boldsymbol{M}^{-1}(\boldsymbol{q})\boldsymbol{S}_a$. Consequently, the control vector fields are state dependent, naturally introducing state-input coupling in a lifted representation and motivating bilinear controlled Koopman models such as \eqref{eq:koopman_bil_con}. Enforcing a constant input matrix in these coordinates instead requires this configuration dependence to be absorbed into the learned finite-dimensional model. The generalized-momentum representation~\cite{singh2025generalized} introduced next removes this configuration dependence from the actuation channel.

\subsection{Generalized-Momentum Representation}

Defining the generalized momentum as $\boldsymbol{p}=\boldsymbol{M}(\boldsymbol{q})\dot{\boldsymbol{q}}$, we introduce the position-momentum state
$\boldsymbol{x} = [\boldsymbol{q}^\top,\, \boldsymbol{p}^\top]^\top\in\mathbb{R}^{2n_q}$. The corresponding mechanical Hamiltonian is 
\begin{equation} 
    \label{eq:physical_H}
    H(\boldsymbol{q},\boldsymbol{p})=\frac{1}{2}\boldsymbol{p}^{\top}\boldsymbol{M}^{-1}(\boldsymbol{q})\boldsymbol{p}+V(\boldsymbol{q}).
\end{equation}

In these coordinates, the Euler-Lagrange dynamics admit the Hamiltonian representation~\cite{porthamschaft}
\begin{equation} 
    \dot{\boldsymbol{x}}=\begin{bmatrix}
        \boldsymbol{0} & \boldsymbol{I}\\
        -\boldsymbol{I} & \boldsymbol{0}
    \end{bmatrix} 
    \nabla_{\boldsymbol{x}}H(\boldsymbol{x}) + \underbrace{
    \begin{bmatrix}
        \boldsymbol{0}\\
        \boldsymbol{S}_a
    \end{bmatrix}}_{\boldsymbol{G}_a} \boldsymbol{u},
    \label{eq:momentum_hamiltonian}
\end{equation}
where
$\boldsymbol{G}_a{\in}\mathbb{R}^{2n_q\times m}$ denotes the constant mechanical actuation port. The generalized-momentum transformation provides two structural advantages. First, it expresses the dynamics in the Hamiltonian variables $(\boldsymbol{q},\boldsymbol{p})$. Second, the configuration-dependent input distribution $\boldsymbol{M}^{-1}(\boldsymbol{q})\boldsymbol{S}_a$ in the position-velocity representation is replaced by the constant mechanical port $\boldsymbol{G}_a=[\boldsymbol{0}^{\top},\boldsymbol{S}_a^{\top}]^{\top}$. This makes the generalized-momentum representation particularly suitable for constructing a controlled Koopman model with an explicitly prescribed actuation channel.

\subsection{Koopman Model in Generalized-Momentum Coordinates}

Let $\boldsymbol{\phi}:\mathbb{R}^{2n_q}\rightarrow\mathbb{R}^{n_\phi}$ denote a set of learnable observable functions. Using the
generalized-momentum state, we define the finite-dimensional lifting
\begin{equation}
    \boldsymbol{z} = \boldsymbol{\Psi}(\boldsymbol{x}) = \begin{bmatrix}
        \boldsymbol{x}\\
        \boldsymbol{\phi}(\boldsymbol{x})
    \end{bmatrix} \in\mathbb{R}^{n_z},\;
    n_z=2n_q+n_\phi.
    \label{eq:lift}
\end{equation}
Since the physical state is included explicitly in the lifted coordinates, it is recovered through the fixed projection $\boldsymbol{x} =\boldsymbol{C}^{x}\boldsymbol{z}$, where $\boldsymbol{C}^{x} = \begin{bmatrix}\boldsymbol{I}_{2n_q} &\boldsymbol{0}_{2n_q\times n_\phi}\end{bmatrix}.$

Following the control-coherent construction~\cite{asada2024control, singh2025generalized}, we consider a continuous-time Koopman predictor
\begin{equation}
    \dot{\boldsymbol{z}}
    =
    \boldsymbol{A}_c\boldsymbol{z}
    +
    \boldsymbol{B}_c\boldsymbol{u},
    \label{eq:koopman_continuous}
\end{equation}
where the actuation structure exposed by the generalized-momentum coordinates is preserved explicitly by prescribing
\begin{equation}
    \boldsymbol{B}_c
    =
    \begin{bmatrix}
        \boldsymbol{0}_{n_q\times m}\\
        \boldsymbol{S}_a\\
        \boldsymbol{0}_{n_\phi\times m}
    \end{bmatrix}
    \in\mathbb{R}^{n_z\times m}.
    \label{eq:Bc}
\end{equation}
The remaining task is to construct $\boldsymbol{A}_c$ to capture the nonlinear dynamics while preserving the underlying interconnection and dissipation structure. This motivates the port-Hamiltonian parameterization developed next.

\section{Port-Hamiltonian Koopman Learning}
\label{sec:ph_koopman}

This section introduces the port-Hamiltonian representation imposed on the finite-dimensional Koopman generator and the corresponding neural architecture used to learn it.

\subsection{Port-Hamiltonian Structure}

A port-Hamiltonian system~\cite{porthamschaft} can be represented as
\begin{equation}
 \dot{\boldsymbol{\xi}} =\left(\boldsymbol{\mathcal{J}}(\boldsymbol{\xi})-\boldsymbol{\mathcal{R}}(\boldsymbol{\xi}) \right)\nabla H(\boldsymbol{\xi})+\boldsymbol{\mathcal{G}}(\boldsymbol{\xi})\boldsymbol{u},\label{eq:ph_general}
\end{equation}
where
$\boldsymbol{\mathcal{J}}^\top=-\boldsymbol{\mathcal{J}}$ describes
the conservative interconnection,
$\boldsymbol{\mathcal{R}}^\top=\boldsymbol{\mathcal{R}}\succeq0$
represents dissipation,
$H$ denotes the Hamiltonian, and
$\boldsymbol{\mathcal{G}}$ defines the external input port.
Defining the power-conjugate output as $ \boldsymbol{y} = \boldsymbol{\mathcal{G}}^\top\nabla H,$
the corresponding energy balance is
\begin{equation}
    \dot{H}
    =
    -
    \nabla H^\top
    \boldsymbol{\mathcal{R}}
    \nabla H
    +
    \boldsymbol{y}^{\top}\boldsymbol{u}
    \leq
    \boldsymbol{y}^{\top}\boldsymbol{u}.
    \label{eq:ph_power_balance}
\end{equation}

Motivated by this structure, and by the operator-level connection between port-Hamiltonian systems and Koopman generators established in~\cite{preciado2026koopman}, we impose an analogous storage-interconnection-dissipation structure on the finite-dimensional Koopman generator while retaining the known mechanical input port in~\eqref{eq:Bc}.

 \subsection{Lifted Storage and Structured Koopman Generator}

We associate the lifted coordinates with the quadratic storage function
\begin{equation}
    \tilde{H}(\boldsymbol{z}) = \frac{1}{2}\boldsymbol{z}^{\top}\boldsymbol{S}\boldsymbol{z},\;\boldsymbol{S}=\boldsymbol{S}^{\top}\succ0,
    \label{eq:lifted_storage}
\end{equation}
with $\nabla_{\boldsymbol{z}}\tilde{H}(\boldsymbol{z}) =\boldsymbol{S}\boldsymbol{z}$. The lifted storage $\tilde{H}$ is not assumed to coincide with the physical Hamiltonian of the Euler-Lagrange system. Instead, $\boldsymbol{S}$ defines a quadratic storage metric in the finite-dimensional Koopman space. The continuous-time Koopman generator is parameterized as
\begin{equation}
\boldsymbol{A}_c=\boldsymbol{S}^{-1}\left(\boldsymbol{J}-\boldsymbol{R}\right)
    \label{eq:Ac_ph}
\end{equation}
where $\boldsymbol{J}^{\top}=-\boldsymbol{J}$ and $\boldsymbol{R}^{\top}=\boldsymbol{R}\succeq0$.  
The parameterization in~\eqref{eq:Ac_ph} admits an equivalent port-Hamiltonian representation.  Defining $\boldsymbol{J}_{\mathrm{pH}}=\boldsymbol{S}^{-1}\boldsymbol{J}\boldsymbol{S}^{-1}, \boldsymbol{R}_{\mathrm{pH}}=\boldsymbol{S}^{-1}\boldsymbol{R}\boldsymbol{S}^{-1}$
yields
\begin{equation}
    \boldsymbol{A}_c\boldsymbol{z}=\left(\boldsymbol{J}_{\mathrm{pH}}-\boldsymbol{R}_{\mathrm{pH}}\right)\nabla_{\boldsymbol{z}}\tilde{H}(\boldsymbol{z}),
\end{equation}
where $\boldsymbol{J}_{\mathrm{pH}}^{\top}=-\boldsymbol{J}_{\mathrm{pH}}$ and $\boldsymbol{R}_{\mathrm{pH}}\succeq0$. Although the two representations are equivalent, we use~\eqref{eq:Ac_ph} since it expresses the storage-dissipation relation directly in terms of the parameterized matrix $\boldsymbol{R}$, simplifying the subsequent analysis.

\begin{remark}
A related port-Hamiltonian Koopman formulation is developed in~\cite{preciado2026koopman}, where the generator decomposition is approximated through weak Galerkin projection onto a prescribed dictionary. The present work differs in three key aspects. First, while the physical port map in~\cite{preciado2026koopman} may be state dependent, generalized-momentum coordinates expose a state-independent mechanical port, allowing $\boldsymbol{B}_c$ to be prescribed directly from the known actuation structure. Second, the observables and structured generator are learned jointly from trajectory data for robotic mechanical systems. Third, the learned continuous-time model is mapped to a discrete Koopman predictor using a Cayley-midpoint discretization that preserves an exact discrete storage balance.
\end{remark}

\subsection{Learnable Parameterization}
To enforce the required structural properties throughout training, the matrices $\boldsymbol{S}$, $\boldsymbol{J}$, and $\boldsymbol{R}$ are
parameterized directly as 
\begin{align}
\label{eq:LKW_parameterization}
    \boldsymbol{S} {=} \boldsymbol{L}\boldsymbol{L}^{\top} {+} \epsilon_s\boldsymbol{I}_{n_z},\; \boldsymbol{J} {=}\boldsymbol{K}{-}\boldsymbol{K}^{\top}, \;\boldsymbol{R}{=}\boldsymbol{W}\boldsymbol{W}^{\top}{+}\epsilon_d\boldsymbol{I}_{n_z},
\end{align}
where $\epsilon_s>0$ and $\epsilon_d\geq0$. Here, $\boldsymbol{L}\in\mathbb{R}^{n_z\times n_z}$ is a learnable lower-triangular matrix, $\boldsymbol{K}\in\mathbb{R}^{n_z\times n_z}$ is unconstrained, and $\boldsymbol{W}\in\mathbb{R}^{n_z\times r}$ is learnable with $r\leq n_z$. These parameterizations ensure $\boldsymbol{S}\succ0$, $\boldsymbol{J}^{\top}=-\boldsymbol{J}$, and $\boldsymbol{R}\succeq0$ by construction. The parameter $r$ bounds the rank of the learned dissipative component $\boldsymbol{W}\boldsymbol{W}^{\top}$, while $\epsilon_d>0$ ensures $\boldsymbol{R}\succ0$. 

The Koopman generator $\boldsymbol{A}_c$ is therefore not learned directly. Instead, the parameters $\{\boldsymbol{L},\boldsymbol{K},\boldsymbol{W}\}$ are optimized jointly, and $\boldsymbol{A}_c$ is assembled from~\eqref{eq:Ac_ph} at each forward pass, ensuring that the prescribed port-Hamiltonian structure is maintained throughout training.

\subsection{Structure-Preserving Discretization}
\label{sec:discretization}

The port-Hamiltonian structure in~\eqref{eq:Ac_ph} is defined in continuous time, whereas the learned Koopman model is ultimately used in discrete time for prediction and control. A generic numerical discretization, however, need not preserve the corresponding continuous-time dissipative structure.

We therefore discretize~\eqref{eq:koopman_continuous} using the implicit midpoint rule. Assuming that the control input is held constant over the sampling interval $[t_k,t_{k+1})$, the discrete dynamics satisfy
\begin{equation}
    \frac{\boldsymbol{z}_{k+1}-\boldsymbol{z}_{k}}{h} =\boldsymbol{A}_c\bar{\boldsymbol{z}}_k +\boldsymbol{B}_c\boldsymbol{u}_k,\quad\bar{\boldsymbol{z}}_k= \frac{\boldsymbol{z}_{k+1}+\boldsymbol{z}_{k}}{2}.
    \label{eq:implicit_midpoint}
\end{equation}
Since $\boldsymbol{S}\succ0$ and $\boldsymbol{R}\succeq0$, the structured generator in~\eqref{eq:Ac_ph} satisfies $\operatorname{Re}(\lambda_i(\boldsymbol{A}_c))\leq0$ for all $i$. Hence, $\boldsymbol{I}_{n_z}-\frac{h}{2}\boldsymbol{A}_c$ is nonsingular for any $h>0$. Rearranging~\eqref{eq:implicit_midpoint} gives $\boldsymbol{z}_{k+1} = \boldsymbol{A}\boldsymbol{z}_k +\boldsymbol{B}\boldsymbol{u}_k,$  where
\begin{align}
    \boldsymbol{A}&{=}\left(\boldsymbol{I}_{n_z} {-} \frac{h}{2}\boldsymbol{A}_c\right)^{-1}\left(\boldsymbol{I}_{n_z} + \frac{h}{2}\boldsymbol{A}_c\right) \label{eq:A_cayley}, \\
    \boldsymbol{B}&{=}h\left(\boldsymbol{I}_{n_z} {-} \frac{h}{2}\boldsymbol{A}_c\right)^{-1} \boldsymbol{B}_c.
    \label{eq:B_midpoint}
\end{align}
The mapping in~\eqref{eq:A_cayley} is the Cayley transform of the continuous-time Koopman generator. Since $\boldsymbol{B}_c$ is fixed by the known mechanical actuation structure, the discrete-time input matrix $\boldsymbol{B}$ is determined analytically from $\boldsymbol{A}_c$, $\boldsymbol{B}_c$, and $h$, rather than being independently identified from data. The structure-preserving properties of this discretization are established in Section~\ref{sec:guarantees}.

\subsection{Learning Objective}
\label{sec:learning_objective}
Having defined the structure-preserving discrete-time predictor, we next describe its data-driven identification. A neural encoder maps the generalized-momentum state $\boldsymbol{x}_k$ to nonlinear observables $\boldsymbol{\phi}_{\theta}(\boldsymbol{x}_k)$, which are appended to the physical state to form $\boldsymbol{z}_k=\boldsymbol{\Psi}_{\theta}(\boldsymbol{x}_k)$. The encoder parameters $\theta$ and the port-Hamiltonian parameters $\boldsymbol{\Theta}_{ph} = \{\boldsymbol{L},\boldsymbol{K},\boldsymbol{W}\}$ are learned jointly from transition data $\{(\boldsymbol{x}_k,\boldsymbol{u}_k,\boldsymbol{x}_{k+1})\}_{k=1}^{N_d}$. 

For each transition, the lifted state and control input are propagated through the discrete Koopman model to obtain the one-step predictions
\begin{equation}
    \hat{\boldsymbol{z}}_{k+1}=\boldsymbol{A}\boldsymbol{z}_k+\boldsymbol{B}\boldsymbol{u}_k,
    \;\hat{\boldsymbol{x}}_{k+1} = \boldsymbol{C}^{x}\hat{\boldsymbol{z}}_{k+1}.
    \label{eq:lifted_prediction}
\end{equation}
The model is trained by penalizing the prediction errors in both the physical-state and learned-observable coordinates. The physical-state prediction loss is defined as
\begin{equation}
    \mathcal{L}_{x}=\frac{1}{N_d}\sum_{k=1}^{N_d}\left\|\hat{\boldsymbol{x}}_{k+1}-\boldsymbol{x}_{k+1}\right\|_2^2. \label{eq:state_prediction_loss}
\end{equation}
To maintain consistency of the learned observables under the lifted dynamics, we additionally define the lifted-space consistency loss as
\begin{equation}
    \mathcal{L}_{\phi} =\frac{1}{N_d}\sum_{k=1}^{N_d}\left\|\hat{\boldsymbol{\phi}}_{k+1}-\boldsymbol{\phi}_{\theta}(\boldsymbol{x}_{k+1})\right\|_2^2,
    \label{eq:lifted_prediction_loss}
\end{equation}
where $\hat{\boldsymbol{\phi}}_{k+1}$ denotes the learned observer component of $\hat{\boldsymbol{z}}_{k+1}$. The overall training loss is
\begin{equation}
    \mathcal{L}=\mathcal{L}_{x}+\alpha_{\phi}\mathcal{L}_{\phi}+\lambda_1\|\boldsymbol{\Theta}_{ph}\|_1 + \lambda_2\|\boldsymbol{\Theta}_{ph}\|_F^2 ,
    \label{eq:total_training_loss}
\end{equation}
where $\alpha_{\phi}>0$ weights the lifted-space consistency loss. $\lambda_1,\lambda_2\geq0$ weight $L_1$ and $L_2$ regularization, respectively.

The discrete matrices $\boldsymbol{A}$ and $\boldsymbol{B}$ are differentiable functions of the structured continuous-time generator $\boldsymbol{A}_c$ through~\eqref{eq:A_cayley} and \eqref{eq:B_midpoint}. Hence, gradients of \eqref{eq:total_training_loss} propagate through the discretization to the learnable matrices $\boldsymbol{\Theta}_{ph}$ and the encoder $\boldsymbol{\phi}_{\theta}$. The complete learning architecture is shown in Fig.~\ref{fig:neural_ph}.
\begin{figure*}
    \centering
    \includegraphics[width=\textwidth]{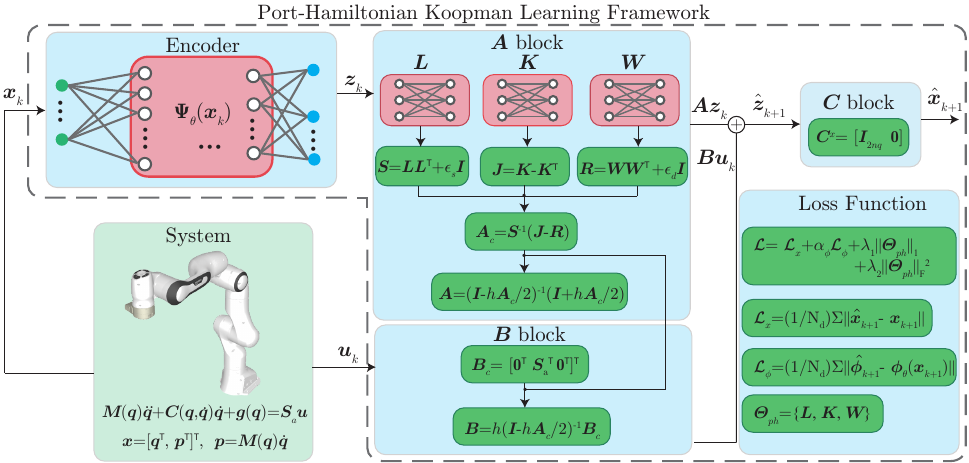}
    \caption{Neural architecture for learning the port-Hamiltonian
    Koopman model.}
    \label{fig:neural_ph}
\end{figure*}

Unlike existing approaches based on constrained dissipativity~\cite{hara2020learning}, stable Koopman parameterizations~\cite{10091950,9867865,sakibNoise}, or post-hoc model correction~\cite{yu2025dissapative}, the proposed formulation enforces the port-Hamiltonian storage, interconnection, and dissipation structure directly by construction. Therefore, no additional stability or passivity constraint is required in the training objective, and minimizing~\eqref{eq:total_training_loss}  identifies the predictive model while maintaining the prescribed port-Hamiltonian structure throughout training. 

\section{Dissipativity and Stability Guarantees}
\label{sec:guarantees}

The structured generator in~\eqref{eq:Ac_ph} inherits the continuous-time dissipativity of finite-dimensional port-Hamiltonian Koopman representations~\cite{preciado2026koopman}, as formalized below.

\begin{proposition}[Continuous-Time Lifted Dissipativity]
\label{prop:continuous_dissipativity}
Consider the lifted dynamics in~\eqref{eq:koopman_continuous} with the
structured generator in~\eqref{eq:Ac_ph}, where $\boldsymbol{S}{\succ}0$, $\boldsymbol{J}^{\top}{=}{-}\boldsymbol{J}$, and $\boldsymbol{R}\succeq0$.
For the storage function in~\eqref{eq:lifted_storage}, define
\begin{equation}
    \tilde{\boldsymbol{y}}
    =
    \boldsymbol{B}_c^{\top}
    \boldsymbol{S}\boldsymbol{z}.
    \label{eq:lifted_output}
\end{equation}
Then,
\begin{equation}
    \dot{\tilde{H}}
    =
    -\boldsymbol{z}^{\top}
    \boldsymbol{R}
    \boldsymbol{z}
    +
    \tilde{\boldsymbol{y}}^{\top}\boldsymbol{u}
    \leq
    \tilde{\boldsymbol{y}}^{\top}\boldsymbol{u}.
    \label{eq:continuous_dissipation}
\end{equation}
Hence, the lifted dynamics are dissipative with respect to the supply
rate $\tilde{\boldsymbol{y}}^{\top}\boldsymbol{u}$.
\end{proposition}

\begin{proof}
From~\eqref{eq:Ac_ph},
\begin{equation}
    \boldsymbol{A}_c^{\top}\boldsymbol{S}
    +
    \boldsymbol{S}\boldsymbol{A}_c
    =
    -2\boldsymbol{R}
    \preceq0.
    \label{eq:continuous_lyapunov_identity}
\end{equation}
Moreover,
$\boldsymbol{S}\boldsymbol{A}_c
=\boldsymbol{J}-\boldsymbol{R}$.
Using
$\nabla_{\boldsymbol{z}}\tilde{H}
=\boldsymbol{S}\boldsymbol{z}$,
\begin{align}
    \dot{\tilde{H}}=
    \boldsymbol{z}^{\top}
    \boldsymbol{S}\dot{\boldsymbol{z}} = 
    \boldsymbol{z}^{\top}
    (\boldsymbol{J}-\boldsymbol{R})
    \boldsymbol{z}
    +
    \boldsymbol{z}^{\top}
    \boldsymbol{S}\boldsymbol{B}_c\boldsymbol{u}.
\end{align}
Since $\boldsymbol{J}^{\top}{=}{-}\boldsymbol{J}$, $\boldsymbol{z}^{\top}\boldsymbol{J}\boldsymbol{z}{=}0$.
Using~\eqref{eq:lifted_output} and $\boldsymbol{R}{\succeq}0$ gives~\eqref{eq:continuous_dissipation}.
\end{proof}

\begin{remark}[Interpretation of the Lifted Output]
The output
$\tilde{\boldsymbol{y}}
=\boldsymbol{B}_c^{\top}\boldsymbol{S}\boldsymbol{z}$
is power conjugate to $\boldsymbol{u}$ with respect to the learned
storage $\tilde{H}$. It need not coincide with the physical port output
$\boldsymbol{S}_a^{\top}\dot{\boldsymbol{q}}$, since $\tilde{H}$ is not
required to equal the mechanical Hamiltonian.
\end{remark}

\begin{theorem}[Dissipativity Preservation]
\label{thm:discrete_dissipation}
Under the conditions of
Proposition~\ref{prop:continuous_dissipativity}, the Cayley-midpoint
discretization in~\eqref{eq:implicit_midpoint} preserves the lifted
dissipativity property with respect to the same storage function
$\tilde{H}$. Specifically,
\begin{equation}
    \tilde{H}(\boldsymbol{z}_{k+1})
    -
    \tilde{H}(\boldsymbol{z}_{k})
    =
    -h\bar{\boldsymbol{z}}_k^{\top}
    \boldsymbol{R}
    \bar{\boldsymbol{z}}_k
    +
    h\bar{\boldsymbol{y}}_k^{\top}
    \boldsymbol{u}_k,
    \label{eq:exact_discrete_balance}
\end{equation}
where
$\bar{\boldsymbol{z}}_k{=}(\boldsymbol{z}_{k+1}{+}\boldsymbol{z}_{k})/2$
and $\bar{\boldsymbol{y}}_k{=}\boldsymbol{B}_c^{\top}\boldsymbol{S}\bar{\boldsymbol{z}}_k$.
Consequently,
\begin{equation}
    \tilde{H}(\boldsymbol{z}_{k+1})
    -
    \tilde{H}(\boldsymbol{z}_{k})
    \leq
    h\bar{\boldsymbol{y}}_k^{\top}\boldsymbol{u}_k.
    \label{eq:discrete_passivity}
\end{equation}
\end{theorem}

\begin{proof}
Using
$\tilde{H}(\boldsymbol{z})
=\frac{1}{2}\boldsymbol{z}^{\top}
\boldsymbol{S}\boldsymbol{z}$
and $\boldsymbol{S}=\boldsymbol{S}^{\top}$,
\begin{align}
    \tilde{H}(\boldsymbol{z}_{k+1})
    -\tilde{H}(\boldsymbol{z}_{k})
    &=
    \bar{\boldsymbol{z}}_k^{\top}
    \boldsymbol{S}
    (\boldsymbol{z}_{k+1}-\boldsymbol{z}_k)
    \nonumber\\
    &=
    h\bar{\boldsymbol{z}}_k^{\top}
    (\boldsymbol{J}-\boldsymbol{R})
    \bar{\boldsymbol{z}}_k
    +
    h\bar{\boldsymbol{y}}_k^{\top}\boldsymbol{u}_k
    \nonumber\\
    &=
    -h\bar{\boldsymbol{z}}_k^{\top}
    \boldsymbol{R}
    \bar{\boldsymbol{z}}_k
    +
    h\bar{\boldsymbol{y}}_k^{\top}\boldsymbol{u}_k,
\end{align}
where
$\bar{\boldsymbol{z}}_k^{\top}
\boldsymbol{J}\bar{\boldsymbol{z}}_k=0$
by skew symmetry of $\boldsymbol{J}$.
Since $\boldsymbol{R}\succeq0$,
\eqref{eq:discrete_passivity} follows.
\end{proof}

\begin{corollary}[Passive Contractivity]
\label{cor:passive_contractivity}
For $\boldsymbol{u}_k=\boldsymbol{0}$,
\begin{equation}
    \tilde{H}(\boldsymbol{z}_{k+1})
    -
    \tilde{H}(\boldsymbol{z}_{k})
    =
    -h\bar{\boldsymbol{z}}_k^{\top}
    \boldsymbol{R}
    \bar{\boldsymbol{z}}_k
    \leq0.
\end{equation}
Moreover, the discrete Koopman transition matrix satisfies
$\boldsymbol{A}^{\top}\boldsymbol{S}\boldsymbol{A}
-\boldsymbol{S}\preceq0$, so the zero-input lifted dynamics are
non-expansive in the $\boldsymbol{S}$-weighted norm.
\end{corollary}

\begin{proof}
Let
$\boldsymbol{M}_h
{=}\boldsymbol{I}_{n_z}
{-}\frac{h}{2}\boldsymbol{A}_c$, which is nonsingular for $h{>}0$ as established in
Section~\ref{sec:discretization}.
Using~\eqref{eq:A_cayley} and
\eqref{eq:continuous_lyapunov_identity},
\begin{equation}
    \boldsymbol{A}^{\top}
    \boldsymbol{S}
    \boldsymbol{A}
    -
    \boldsymbol{S}
    =
    -2h
    \boldsymbol{M}_h^{-\top}
    \boldsymbol{R}
    \boldsymbol{M}_h^{-1}
    \preceq0.
    \label{eq:cayley_lyapunov_identity}
\end{equation}
The result follows from $\boldsymbol{R}\succeq0$.
\end{proof}

\begin{remark}[Forward-Euler Discretization]
\label{rem:forward_euler_discr}
For forward-Euler discretization,
$\boldsymbol{A}_{E}
=\boldsymbol{I}_{n_z}+h\boldsymbol{A}_c$.
Using~\eqref{eq:continuous_lyapunov_identity},
\begin{equation}
    \boldsymbol{A}_{E}^{\top}
    \boldsymbol{S}
    \boldsymbol{A}_{E}
    -
    \boldsymbol{S}
    =
    -2h\boldsymbol{R}
    +
    h^2
    \boldsymbol{A}_c^{\top}
    \boldsymbol{S}
    \boldsymbol{A}_c.
    \label{eq:euler_dissipation}
\end{equation}
The positive-semidefinite second term may offset the dissipative term
for finite $h$. Thus, unlike
Corollary~\ref{cor:passive_contractivity}, continuous-time
dissipativity does not generally imply discrete-time contractivity
under forward-Euler discretization.
\end{remark}

\begin{figure*}[!t]
    \centering
    \includegraphics[width=\linewidth]{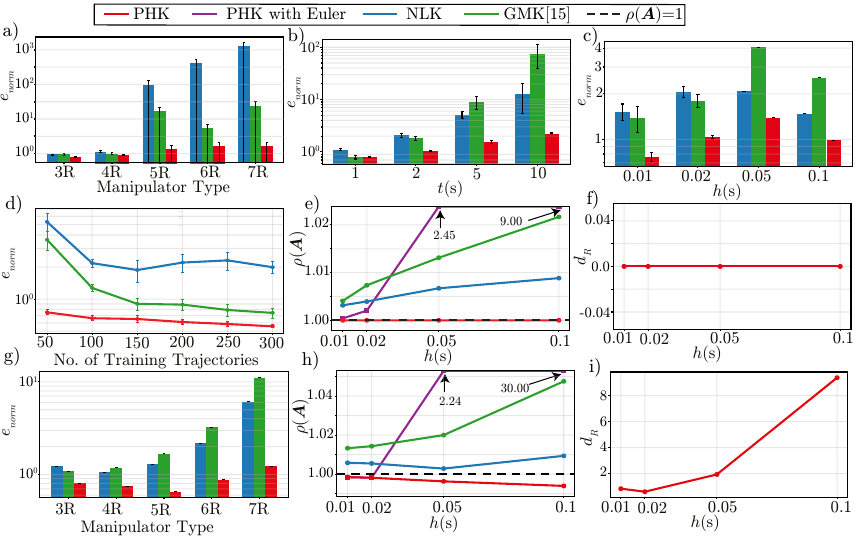}
    \caption{Open-loop prediction and structural characteristics of PHK (proposed), NLK and GMK~\cite{singh2025generalized} for ideal (a-f) and damped (g-i) manipulators. Ideal systems: normalized prediction error ($e_{\mathrm{norm}}$) a) across manipulator dimensions at $h=0.02$~s and $t=2$~s; b) across prediction horizons for 5R at $h=0.02$~s; c) across sampling periods for 5R at $t=2$~s; d) data efficiency for 5R at $h=0.02$~s and $t=2$~s; e) spectral radius of the learned 5R models, with out-of-range values annotated; and f) mean storage-normalized dissipation of PHK for 5R. Damped systems: g) normalized prediction error across manipulator dimensions at $h=0.05$~s and $t=2$~s; h) spectral radius of the learned 7R models, with out-of-range values annotated; and i) mean storage-normalized dissipation of PHK for 7R.}
    \label{fig:open_loop}
\end{figure*}

\begin{corollary}[Strict Dissipativity and Schur Stability]
\label{cor:strict_dissipativity}
If $\epsilon_d>0$ in~\eqref{eq:LKW_parameterization}, then
$\boldsymbol{R}
\succeq
\epsilon_d\boldsymbol{I}_{n_z}
\succ0$, and
\begin{align}
    &\dot{\tilde{H}}
    \leq
    \tilde{\boldsymbol{y}}^{\top}\boldsymbol{u}
    -
    \epsilon_d\|\boldsymbol{z}\|_2^2,
    \label{eq:strict_continuous_dissipation}\\
    &\tilde{H}(\boldsymbol{z}_{k+1})
    -
    \tilde{H}(\boldsymbol{z}_{k})
    \leq
    h\bar{\boldsymbol{y}}_k^{\top}\boldsymbol{u}_k
    -
    h\epsilon_d\|\bar{\boldsymbol{z}}_k\|_2^2.
    \label{eq:strict_discrete_dissipation}
\end{align}
Moreover,
$\boldsymbol{A}^{\top}\boldsymbol{S}\boldsymbol{A}
-\boldsymbol{S}\prec0$,
and therefore $\rho(\boldsymbol{A})<1$.
\end{corollary}

\begin{proof}
The two dissipation inequalities follow directly from
$\boldsymbol{R}
\succeq
\epsilon_d\boldsymbol{I}_{n_z}$, \eqref{eq:continuous_dissipation} and \eqref{eq:exact_discrete_balance}. Since $\boldsymbol{R}\succ0$ and  $\boldsymbol{M}_h$ is nonsingular, \eqref{eq:cayley_lyapunov_identity} yields $\boldsymbol{A}^{\top}\boldsymbol{S}\boldsymbol{A}-\boldsymbol{S}\prec0.$ As $\boldsymbol{S}\succ0$, the strict discrete Lyapunov inequality implies that $\boldsymbol{A}$ is Schur stable, and hence $\rho(\boldsymbol{A})<1$.
\end{proof}

\begin{remark}[Underactuated Systems] For $m<n_q$,
the constant actuation distribution $\boldsymbol{S}_a$ is retained
through $\boldsymbol{B}_c$ in~\eqref{eq:Bc}, without altering the
structured generator or its dissipativity guarantees.
\end{remark}

\section{Results}
\label{sec:results}
This section evaluates the proposed architecture in simulation and hardware for open-loop prediction and closed-loop tracking.

\subsection{Open-Loop Prediction and Spectral Stability}
\label{sec:open_loop_results}
We evaluate the open-loop prediction performance of the proposed Port-Hamiltonian Koopman (PHK) model against two Koopman baselines: a nominal linear Koopman model (NLK) using $[\boldsymbol{q}^{\top},\dot{\boldsymbol{q}}^{\top}]^{\top}$ and the generalized-momentum Koopman model (GMK)~\cite{singh2025generalized} using $[\boldsymbol{q}^{\top},\boldsymbol{p}^{\top}]^{\top}$. GMK~\cite{singh2025generalized} is the closest baseline since it uses the same generalized-momentum coordinates as PHK but without the port-Hamiltonian structure. We consider two settings: ideal manipulators without damping and manipulators with dissipative effects. The former isolates the ability of the learned structure to prevent artificial energy growth, while the latter evaluates its ability to represent physical dissipation through $\boldsymbol{R}$. For each manipulator, all methods use the same encoder architecture, lift dimension, and training trajectories, with $n_\phi$ scaled with $n_q$. Prediction error, denoted by $e_{\mathrm{norm}}$, is the state RMSE normalized component-wise by the corresponding test-set standard deviation. Results are averaged over individual runs, with error bars denoting one standard deviation. For PHK, we additionally report the mean storage-normalized dissipation $d_R= \frac{1}{n_z} \operatorname{tr}\!\left(\boldsymbol{S}^{-1/2}\boldsymbol{R}\boldsymbol{S}^{-1/2}\right),$ which quantifies the average dissipation strength relative to the learned storage metric. Thus, $d_R = 0$ indicates nearly conservative lifted dynamics, while larger values indicate stronger learned dissipation.

\begin{figure*}[!t]
    \centering
    \includegraphics[width=\linewidth]{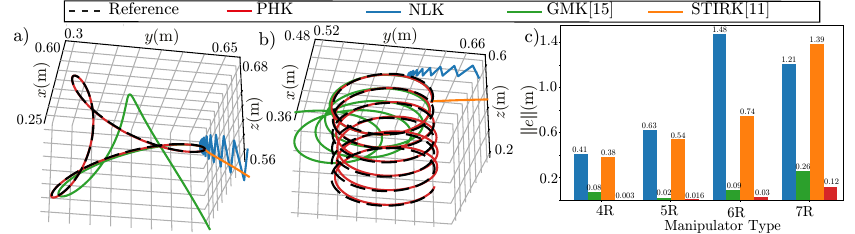}
    \caption{Closed-loop tracking performance for PHK, NLK, GMK~\cite{singh2025generalized}, and STIRK~\cite{sakibNoise}. 4R end-effector tracking for a) hypotrochoid and b) helix references at $h=0.05$~s; c) aggregate Cartesian tracking RMSE across manipulator dimensions, averaged over the hypotrochoid, petal, and helix trajectories and sampling periods $h\in\{0.01,\,0.02,\,0.05\}$~s.}
    \label{fig:closed_loop}
\end{figure*}


Figures~\ref{fig:open_loop}a-f summarize the results for ideal manipulators. PHK consistently achieves lower prediction error than NLK and GMK~\cite{singh2025generalized}, with the gap increasing with manipulator dimension (Fig.~\ref{fig:open_loop}a), prediction horizon (Fig.~\ref{fig:open_loop}b), and sampling period (Fig.~\ref{fig:open_loop}c). PHK also exhibits improved data efficiency (Fig.~\ref{fig:open_loop}d), reaching near-converged performance with approximately 150 training trajectories, compared with about 300 for GMK~\cite{singh2025generalized}, while NLK remains unconverged over the considered range. Since GMK~\cite{singh2025generalized} and PHK use the same generalized-momentum coordinates, their difference in performance isolates the effect of the port-Hamiltonian structure. The spectral results in Fig.~\ref{fig:open_loop}e further explain these trends. By Corollary~\ref{cor:passive_contractivity}, PHK satisfies $\rho(\boldsymbol{A})\leq1$, whereas the unconstrained baselines learn models with $\rho(\boldsymbol{A})>1$. The resulting unstable modes are amplified during recursive prediction, consistent with the increasing error at longer horizons and larger sampling periods. The forward-Euler realization of the same learned continuous-time PHK generator also becomes unstable as $h$ increases, while the Cayley-midpoint realization remains non-expansive, consistent with Remark~\ref{rem:forward_euler_discr}. Moreover, Fig.~\ref{fig:open_loop}f shows that $d_R$ remains close to zero, indicating that PHK introduces negligible artificial dissipation into the conservative dynamics.

Figures~\ref{fig:open_loop}g-i show the corresponding results for damped manipulators. PHK again achieves the lowest prediction error as the manipulator dimension increases (Fig.~\ref{fig:open_loop}g) and remains non-expansive across the tested sampling periods (Fig.~\ref{fig:open_loop}h), whereas the unconstrained baselines and the forward-Euler realization can yield unstable discrete-time models. In contrast to the ideal case, Fig.~\ref{fig:open_loop}i shows a clearly nonzero $d_R$, indicating that the learned $\boldsymbol{R}$ captures dissipative behavior when present in the underlying dynamics. Together, these results show that PHK provides accurate and data-efficient prediction while preserving the distinction between conservative and dissipative dynamics and maintaining spectrally well-behaved discrete-time predictors.


\subsection{Closed-Loop Tracking Performance}
\label{sec:closed_loop}
The closed-loop tracking performance is evaluated on ideal manipulators without damping using NLK, GMK~\cite{singh2025generalized}, and STIRK~\cite{sakibNoise} as baselines. STIRK~\cite{sakibNoise} provides a stability-constrained Koopman baseline, allowing us to assess whether stable lifted dynamics alone are sufficient for accurate control. Each learned model is embedded within the an MPC framework implemented using ACADOS~\cite{acadosverschueren2021}. The closed loop performance is evaluated across a range of reference trajectories (hypotrochoid, petal, and helix), manipulator dimensions (4R--7R), and sampling periods $h\in\{0.01,\,0.02,\,0.05\}$~s. GMK~\cite{singh2025generalized} is originally deployed with a generalized extended state observer (GESO) to compensate for lumped model mismatch and disturbances. In the present simulations, however, no observer is used for any method, allowing the closed-loop comparison to reflect differences in the learned predictive models directly.

Fig.~\ref{fig:closed_loop} summarizes the closed-loop results. The representative hypotrochoid and helix trajectories in Figs.~\ref{fig:closed_loop}a and \ref{fig:closed_loop}b for 4R at $h=0.05$~s show that PHK follows the reference paths more closely than the baselines. Figure~\ref{fig:closed_loop}c reports the Cartesian tracking RMSE averaged across all three reference trajectories and sampling periods, showing that PHK maintains the lowest Cartesian tracking error with increase in the improvement in performance as the manipulator dimension increases. Although STIRK~\cite{sakibNoise} enforces stability of the lifted dynamics, its larger tracking errors indicates that stability alone is insufficient to preserve the underlying energy structure, which can limit model accuracy and consequently degrade control performance for the ideal mechanical systems considered here. In contrast, PHK combines stable prediction with the port-Hamiltonian distinction between conservative interconnection and dissipation. Consistent with the open-loop results, the increasing performance gap with manipulator dimension indicates that these structural properties become increasingly important for accurate prediction and closed-loop control as system complexity grows.

\subsection{Experimental Validation}
\label{sec:franka_experiments}

\begin{figure*}[!t]
    \centering
    \includegraphics[width=\linewidth]{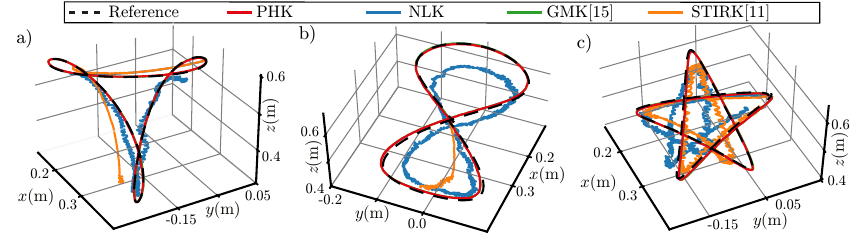}
    \caption{Experimental end-effector tracking on the Franka FR3 hardware using PHK, NLK, GMK~\cite{singh2025generalized} and STIRK~\cite{sakibNoise} for a) hypotrochoid, b) lissajous, and c) star trajectories. STIRK~\cite{sakibNoise} fails to complete the tracking tasks for the hypotrochoid and lissajous trajectories.}
    \label{fig:franka_real_tracking}
\end{figure*}

The framework is further validated on the 7-DoF Franka FR3. All methods use the same ACADOS-based MPC formulation, cost weights, prediction horizon, state and input constraints, and generalized extended state observer (GESO), following~\cite{singh2025generalized}, to compensate for unmodeled dynamics and sim-to-real mismatch. The sampling time is taken to be $h = 0.01$~s.

As shown in Fig.~\ref{fig:franka_real_tracking}, PHK and GMK~\cite{singh2025generalized} closely track all three references, whereas NLK exhibits substantially larger deviations. The corresponding Cartesian tracking RMSEs are $\{0.002,\,0.005,\,0.004\}$~m for PHK, $\{0.002,\,0.005,\,0.003\}$~m for GMK, and $\{0.046,\,0.071,\,0.110\}$~m for NLK. STIRK~\cite{sakibNoise} completes only the star trajectory, with a RMSE of $0.037$~m; for the remaining references, the MPC solver becomes infeasible and the runs are terminated. Since STIRK~\cite{sakibNoise} completes all simulation trajectories, this indicates that sim-to-real mismatch further degrades its less accurate
predictor beyond what the controller can accommodate.

Both PHK and STIRK~\cite{sakibNoise} are Schur stable, with $\rho(\boldsymbol{A})=0.997$ and $0.999$, respectively, whereas NLK and GMK~\cite{singh2025generalized} yield $\rho(\boldsymbol{A})=1.002$ and $1.003$. Despite their similar spectral radii, the markedly different tracking performance of NLK and GMK highlights the benefit of the generalized-momentum representation and its explicit mechanical actuation structure. Conversely, STIRK~\cite{sakibNoise} remains Schur stable yet fails two tracking tasks, showing that spectral stability alone does not guarantee sufficient control-relevant model fidelity. PHK combines these complementary properties by retaining the known mechanical actuation port while explicitly separating conservative interconnection from dissipation. The comparable tracking accuracy of PHK and GMK~\cite{singh2025generalized} under GESO indicates that disturbance compensation can mitigate residual model mismatch on hardware, although only PHK retains the nominal stability guarantees of Section~\ref{sec:guarantees}. 

Together with the observer-free simulation results, these experiments show that PHK improves the intrinsic structure of the learned predictor while remaining effective under sim-to-real uncertainty.

\section{Conclusion}
\label{sec:conc}
This work presented a port-Hamiltonian Koopman framework for Euler-Lagrange mechanical systems. By combining generalized-momentum coordinates with a structured finite-dimensional Koopman generator, proposed PHK preserves the known mechanical input port while embedding storage, interconnection, and dissipation directly into the lifted dynamics. A Cayley-midpoint discretization preserves the corresponding dissipativity properties in the discrete predictor. Theoretical analysis established continuous-time dissipativity, exact discrete storage-dissipation balance, zero-input contractivity, and Schur stability under strict dissipation. Simulations across manipulators of increasing dimension demonstrated improved open-loop prediction, data efficiency, spectral stability, and closed-loop tracking over Koopman baselines. Experiments on the Franka FR3 further validated the framework on hardware.

Overall, PHK yields accurate and structurally reliable predictive models, providing a principled framework for data-driven prediction and control of nonlinear mechanical systems.

\bibliographystyle{ieeetr}
\bibliography{citation.bib}
\end{document}